\documentclass[fleqn,10pt]{olplainarticle}

\usepackage{amsmath}
\usepackage{amsfonts}
\usepackage{amssymb}
\usepackage{amsthm}
\usepackage{amscd}
\usepackage{bbm}
\usepackage{boldline}
\usepackage{multirow}

\newtheorem{thm}{Theorem}[section]

\title{A Density Ratio Super Learner}

\author[1]{Wencheng Wu}
\author[1]{David Benkeser\thanks{Corresponding author}}
\affil[1]{Department of Biostatistics and Bioinformatics, Rollins School of Public Health, Emory University, Atlanta, GA, USA}

\begin{abstract}
The estimation of the ratio of two density probability functions is of great interest in many statistics fields, including causal inference. In this study, we develop an ensemble estimator of density ratios with a novel loss function based on super learning. We show that this novel loss function is qualified for building super learners. Two simulations corresponding to mediation analysis and longitudinal modified treatment policy in causal inference, where density ratios are nuisance parameters, are conducted to show our density ratio super learner's performance empirically.
\end{abstract}
\begin{document}

\flushbottom
\maketitle
\thispagestyle{empty}

\section{Introduction}

\noindent The estimation of the ratio of two probability density functions is an important problem in statistics. In some settings, the ratio of densities is of direct scientific interest. For example, density ratios can be used to identify covariate shifts in general machine learning applications \cite{sugiyama_direct_2007}. In other settings, the ratio of densities is of indirect interest. For example, in the field of causal inference, density ratios play a key intermediary role in the identification and estimation of certain causal effects \cite{tchetgen_semiparametric_2012, diaz_nonparametric_2021, zheng_targeted_2012}. Thus, there is considerable motivation to develop flexible and robust estimators of density ratios. Many methods are available in the literature to this end. In the covariate shift literature, non-parametric algorithms for marginal density ratio estimation have been developed using kernel-based approaches \cite{sugiyama_dimensionality_2010, sugiyama_direct_2007, kanamori_least-squares_2009}. In causal inference, it is common to estimate ratios of conditional densities by first re-writing these ratios in terms of conditional odds and then applying standard classification algorithms to estimate the requisite odds \cite{zheng_targeted_2012}. 

In practice, it may be difficult to determine which estimation approach should be used in a given application. Moreover, within a particular estimation approach, there are often challenging modeling decisions that must be made. For example, if classification-based approaches are used, the practitioner must determine which classification algorithm should be utilized and the values of any tuning parameters required by that algorithm. To aid in the process, we aim to develop an ensemble machine learning-based estimator of density ratios that can be applied to the density ratio estimation problem. Our approach is based on the idea of super learning.

Super learning is an approach to building ensembles of machine learning estimators\cite{laan_super_2007}. Given a library of user-specified learners, super learning determines the optimal convex combination of these learners based on their cross-validated performance on a dataset. This final convex combination of individual learners is thus termed a \emph{super learner}. An appealing practical aspect of super learning is that users do not necessarily have to decide on one particular learning approach. Instead, users can simply add all approaches of interest into the learner library to conduct super learning and use the final super learner to give predictions. Concerns regarding overfitting may be mitigated by the fact that the final weight of each learner is decided by minimizing a $v$-fold cross-validated risk, ensuring that super learners are constructed to optimize out-of-sample, rather than in-sample, prediction performance. 

A key step in the super learner process is the selection of the risk criterion that is used to estimate weights for the ensemble\cite{laan_super_2007}. Often risks, are chosen to as the average of a so-called \emph{loss function}, a function that takes as input an estimate of the quantity of interest and a single observation of data and returns a measure of the goodness of fit of the estimator on that single observation. For example, squared error loss is a commonly used loss function in a general regression setting. This loss function computes the squared distance from the regression-based prediction of the outcome on the data point and the actual observed outcome. The expectation of squared error loss is referred to as the mean squared error risk criterion. A key consideration in selecting a loss function and associated risk criterion is to ensure that the loss is a \emph{qualified loss function} in the sense that the true parameter minimizes the risk criterion. Thus, we may regard estimates with smaller risks as closer to the true parameter in some appropriate sense. 

The asymptotic optimality of the super learner is often established via consideration of so-called oracle inequalities\cite{laan_super_2007, vaart_oracle_2006, dudoit_asymptotics_2005}. Such inequalities are established by considering the difference between the performance of the super learner (in terms of the selected risk criteria) and the performance of the so-called oracle, which selects the optimal ensemble based on true out-of-sample performance, rather than based on estimates of cross-validated risk. Assuming a bounded loss function on the true data distribution, we can often establish an upper bound on the difference between the performance of the super learner and this oracle. Often, we are able to show that the ratio of risks converges to 1 as the sample size increases. 

The argument above highlights that in the context of super learning, the selection of appropriate loss function and associated risk criteria is a critical step in the procedure. However, this issue has been largely overlooked in the context of super learners for estimating density ratios. Consequently, there is a pressing need for the development of a novel loss function tailored specifically to density ratio estimation. This new loss function will play a pivotal role in enhancing the accuracy and reliability of density ratio estimation within the super learning framework. 

In the next section, we propose a new loss function designed specifically for density ratio estimation. We also present two causal inference scenarios in which density ratios are incorporated as part of the target causal parameter or as a nuisance parameter needed in an estimator. Two simulations are conducted to evaluate how well the density ratio super learner with our new loss function can perform on density ratios within the scenarios.

\section{Methods}
\subsection{Density Ratio Parameter}
Let $X_1$ and $X_2$ be two sets of random variables, and let $\lambda$ be a discrete random variable, of which the values are contained in a set $\Lambda$. We have random variable $O=(X_1,X_2, \lambda)\sim P_0$ with the probability density function $p_0$. The density ratio parameter of interest is given by $\psi_0(x_1,x_2)=\frac{p_0(x_1\mid x_2,\lambda=\lambda_1)}{p_0(x_1\mid x_2, \lambda=\lambda_2)}$ for some $\lambda_1,\lambda_2\in \Lambda$. Without loss of generalizability, we let $\lambda_1=1, \lambda_2=0,$ and $\Lambda=\left\{0,1\right\}$. Thus, the target density ratio parameter can be written as
\begin{equation}\label{dr_para}
    \psi_0(x_1,x_2)=\frac{p_0(x_1\mid x_2,\lambda=1)}{p_0(x_1\mid x_2, \lambda=0)}
\end{equation}

\subsection{Density Ratios in Causal Inference}\label{dr_in_ci}

Density ratios are of indirect interest in many causal inference settings. Here, we introduce two scenarios where density ratios are nuisance parameters that need to be estimated: causal mediation analysis and longitudinal modified treatment policies (LMTPs) effect assessment.

\subsubsection{Mediation Analysis}

The causal effect of an exposure on an outcome can often be explained by the exposure's impact on some set of mediating variables. Often it is of interest to characterize not only the total causal effect of an exposure, but also the mechansisms whereby it impacts the outcome. Mediation analysis addresses this problem.
The natural direct effect (NDE) is a key parameter in mediation analysis. The natural direct effect represents the effect of the exposure on the outcome when the mediator is set to the level it would have naturally taken if the exposure had not been assigned. In other words, it captures the effect of the exposure on the outcome that is not mediated through the mediator. The complementary part of NDE in the total causal effect is called the natural indirect effect (NIE). These concepts are helpful for understanding the mechanisms of how exposure affects an outcome \cite{imai_identification_2010, tchetgen_semiparametric_2012, zheng_targeted_2012}. Given a binary exposure $A$, a mediator $M$, and an outcome $Y$, the relations among average treatment effect, natural direct effect, and natural indirect effect can be represented as follows. 
\begin{align*}
    \text{ATE}
    &=\mathbb{E}(Y(1)-Y(0))=\mathbb{E}\left[Y(1,M(1))-Y(0, M(0))\right]\\
    &=\mathbb{E}\left[Y(1,M(1))-Y(1, M(0))\right] + \mathbb{E}\left[Y(1,M(0))-Y(0, M(0))\right]\\
    &=\text{NIE} + \text{NDE}
\end{align*}

Here we use $Y(a)$ and $Y(a,m)$ to represent the counterfactual random variable corresponding to the outcome $Y$, had the exposure $A$ been set to level $a$ and the $(A,M)$ been set to level $(a,m)$, respectively.

Under assumptions, these counterfactual parameters can be estimated using observations of $O = (W, A, M, Y)$ for some set of confounding variables $W$ \cite{tchetgen_semiparametric_2012}. For example, the identification of $\mathbb{E}\left[Y(0,M(1)\right]$ is
\begin{equation*}
    \mathbb{E}\left[Y(0,M(1)\right]=\mathbb{E}\Bigl[\mathbb{E}\bigl[\mathbb{E}\left[Y\mid A=0,M,W\right]\mid A=1, W\bigr]\Bigr]
\end{equation*}
A one-step estimator of this quantity is given by
\begin{align}
    \hat{\phi}=\frac{1}{n}\sum_{i=1}^n 
    &\left[\frac{\mathbb{I}(A_i=0)}{\hat{g}(A=0\mid w_i)}\hat{\psi}(m_i,w_i)(y_i-\hat{\mu}_i) \right.\nonumber\\
    &\hspace{4em} \left. +\frac{\mathbb{I}(A_i=1)}{\hat{g}(A=1\mid w_i)}(\hat{\mu}_i-\hat{\theta}_i)+\hat{\theta}_i\right] \ . 
\end{align}
Here $\hat{\psi}(m,w)$ is an estimate of $\frac{p(m\mid w,A=1)}{p(m\mid w,A=0)}$, which is a density ratio parameter $\psi_0$ as we defined in Equation \ref{dr_para}, but with $(X_1,X_2,\lambda)$ replaced with $(M,W,A)$. We have $\psi_0(m,w)=\frac{p(m\mid w,A=1)}{p(m\mid w,A=0)}$. $\hat{g}(a\mid w)$ is an estimate of the propensity score $p(a\mid w)$. $\mu$ is an estimate of the outcome regression $\mu_0$, where $\hat{\mu}_i=\hat{\mu}(m_i, w_i)$, $\mu_0(m,w)=\mathbb{E}(Y\mid A=0,m_i,w_i)$. $\theta$ is an estimate of the sequential regression $\theta_0$, where $\hat{\theta}_i=\hat{\theta}(w_i)$, $\theta_0(w)=\mathbb{E}\left[\mathbb{E}(Y\mid A=0,M,w)\mid A=1, w\right]$. 

The nuisance parameter $\psi_0$ can be rewritten into terms of conditional odds.
\begin{align}\label{2_cl}
    \psi_0(m,w)
    &=\frac{p(m\mid w,A=1)}{p(m\mid w,A=0)}=\frac{p(m\mid w,A=1)}{p(m\mid w, A=0)}\frac{p(w\mid A=1)}{p(w\mid A=0)}\frac{p(w\mid A=0)}{p(w\mid A=1)}\nonumber\\
    &=\frac{p(m,w\mid A=1)}{p(m,w\mid A=0)}\frac{p(w\mid A=0)}{p(w\mid A=1)}\nonumber\\
    &=\frac{p(A=1\mid m,w)}{p(A=0\mid m,w)}\frac{p(A=0\mid w)}{p(A=1\mid w)}
\end{align}
Naturally, $\hat{\psi}(m,w)$ can be obtained through a plug-in estimator, where estimates of $p(a\mid m,w)$ and $p(a\mid w)$ based on some classification algorithms are used to replace the terms. This parameter can also be rewritten into the ratio of two marginal density ratios
\begin{align}
    \psi_0(m,w)
    &=\frac{p(m\mid w,A=1)}{p(m\mid w,A=0)}=\frac{p(m,w\mid A=1)}{p(m,w\mid A=0)}\bigg/\frac{p(w\mid A=1)}{p(w\mid A=0)}\nonumber\\
    &=\frac{p_{A=1}(m,w)}{p_{A=0}(m,w)}\bigg/\frac{p_{A=1}(w)}{p_{A=0}(w)}
\end{align}
Here, we use $p_{A=a}$ to denote the probability density function of the distribution of $(M,W)$ given $A=a$. Thus, an estimate $\hat{\psi}$ of the density ratio can also be obtained by replacing the two marginal density ratios with their corresponding estimates. Some kernel methods can estimate marginal density ratios directly \cite{sugiyama_dimensionality_2010, sugiyama_direct_2007}. 

In our density ratio super learner, we will include the aforementioned two kinds of plug-in estimators in the library to enhance the estimation process, and they are referred to as classification-based learners and kernel-based learners, respectively.

\subsubsection{Longitudinal Modified Treatment Policy}\label{method:lmtp}

When an exposure is continuous or can have many values, dose-response analysis is traditionally used to estimate the causal effect. In a dose-response analysis, researchers investigate the expected outcome under various hypothetical scenarios, each of which assumes that the whole population is exposed to the same level of the exposure variable. Although informative, this approach has certain limitations including being dependent on correctly-specified parametric models, or relying on relatively slow rates of convergence if using non-parametric methods \cite{neugebauer_nonparametric_2007, kennedy_nonparametric_2017}. To provide an alternative formulation of causal effects in these settings, researchers have proposed so-called modified treatment policies.

Under a modified treatment policy, the post-intervention exposure levels can depend on the observed exposure levels \cite{munoz_population_2012, diaz_nonparametric_2021}. For example, if we are interested in quantifying a causal effect of daily exercise time in minutes $A$ on some health outcome $Y$, rather than targeting counterfactual means $\mathbb{E}\left[Y(a)\right]$ for $a$ ranging over some interval $\mathcal{A}$, we can instead hypothesize a world where all the individual will prolong their current daily exercise time by 10 minutes. Under the modified treatment policies framework, it is said that we have a modified treatment policy of $\mathbbm{d}(a)=a+10$. Now, we use $Y(A^\mathbbm{d})$ to represent the counterfactual random variable corresponding $Y$ had the population prolonged their daily exercise time. The effect of daily exercise time can then be quantified as $\mathbb{E}\left[Y(A^\mathbbm{d})\right]-\mathbb{E}(Y)$. 

Longitudinal modified treatment policies extend the notion of modified treatment policies to longitudinal setting where exposures are measured at multiple points in time. Let $Z_i,i\in\left\{1,2,\ldots,n\right\}$ denote iid observations of $Z=(W_1,A_1,W_2$,$A_2,\ldots,$ $W_\tau,A_\tau,Y)\sim P$. $W_t$ and $A_t$ are time-varying covariates and exposures at time $t$, $t\in\left\{1,2,\ldots,\tau\right\}$. We use $\bar{X}_t=(X_1,\ldots,X_t)$ to denote the history variable, and introduce $H_t=(W_t,\bar{A}_{t-1})$ to represent the history of all variables before $A_t$. $H_t(\bar{A}^\mathbbm{d}_{t-1})$ and $A_t(\bar{A}^\mathbbm{d}_{t-1})$ are counterfactual variables of $H_t$ and $A_t$, with all $A_k$ set to some other random variable $A^\mathbbm{d}_k$, $k\in\left\{1,\ldots,t-1\right\}$. An intervention $A_t^\mathbbm{d}$ is a longitudinal modified treatment policy if it is defined as $A_t^\mathbbm{d}=\mathbbm{d}\bigl(A_t(\bar{A}^\mathbbm{d}_{t-1}), H_t(\bar{A}^\mathbbm{d}_{t-1})\bigr)$. For example, $\mathbbm{d}(a_t,h_t)=a_t+10$. We are interested in $\mathbb{E}\left[Y(A^\mathbbm{d})\right]$. It is demonstrated that under proper assumptions, 
\begin{equation*}
    \mathbb{E}\left[Y(A^\mathbbm{d})\right]=\mathbb{E}\left[\left(\prod_{t=1}^\tau r_t(A_t,H_t)\right)Y\right]
\end{equation*}
where $r_t(a_t,h_t)=\frac{g_t^\mathbbm{d}(a_t\mid h_t)}{g_t(a_t\mid h_t)}$ is a density ratio \cite{young_identification_2014, diaz_nonparametric_2021}. $g_t(a_t\mid h_t)$ and $g_t^\mathbbm{d}(a_t\mid h_t)$ are the propensity scores in the observed data and the hypothesized intervened world, respectively. A plug-in estimator of this is thus
\begin{equation}
    \hat{\phi}=\frac{1}{n}\sum_{i=1}^n\left(\prod_{t=1}^\tau \hat{r}_t(A_t,H_t)\right)Y_i
\end{equation}

A novel approach to generating an estimate $\hat{r}_t$ is proposed in \cite{diaz_nonparametric_2021}. The authors propose to copy the whole dataset for a given time point $t$ and replace $A_t$ with $A_t^\mathbbm{d}$ in the copied dataset. With an indicator variable $\lambda$ (an observation with $\lambda=1$ comes from the copied dataset, while $\lambda=0$ means it is from the original dataset), we merge the two datasets and denote the corresponding distribution as $P^\lambda$. Now, since the conditional distribution of $A_t$ given $H_t$ aligns with the propensity score $g^\mathbbm{d}(a_t\mid h_t)$, the density ratio can be rewritten as
\begin{align*}
    r_t(a_t,h_t)
    =\frac{g_t^\mathbbm{d}(a_t\mid h_t)}{g_t(a_t\mid h_t)}&=\frac{p^\lambda(a_t\mid h_t, \lambda=1)}{p^\lambda(a_t\mid h_t, \lambda=0)}
\end{align*}
Thus, this $r_t(a_t,h_t)$ is another density ratio parameter $\psi_0$ as we defined in Equation \ref{dr_para}, with $(X_1,X_2,\lambda)$ replaced with $(A_t,H_t,\lambda)$. We have $\psi_0(a_t,h_t)=\frac{p^\lambda(a_t\mid h_t, \lambda=1)}{p^\lambda(a_t\mid h_t, \lambda=0)}$. This term can be rewritten into a conditional odd
\begin{align}
    \psi_{0}(a_t,h_t)
    &=\frac{p^\lambda(a_t\mid h_t, \lambda=1)}{p^\lambda(a_t\mid h_t, \lambda=0)}=\frac{p^\lambda(a_t\mid h_t, \lambda=1)p^\lambda(h_t,\lambda =1)p^\lambda(h_t,a_t)}{p^\lambda(a_t\mid h_t, \lambda=0)p^\lambda(h_t,\lambda =0)p^\lambda(h_t,a_t)}\nonumber\\
    &=\frac{p^\lambda(\lambda=1\mid a_t,h_t)}{p^\lambda(\lambda=0\mid a_t,h_t)}
\end{align}
An estimate $\hat{\psi}$ can thus be obtained by replacing $p^\lambda$ with estimates based on some classification algorithms. The step of augmenting the dataset is, in essence, constructing a pseudo-sample of equal weight from the distribution of $(A_t,H_t)$ with a probability density function $f(a_t,h_t)=g_t^\mathbbm{d}(a_t\mid h_t)p(h_t)$, where $p(h_t)$ is the probability density function of $H_t$ in the observed data distribution. In this sense, the conditional density ratio $r_t$ can be transformed into a marginal density ratio
\begin{align}
    \psi_0(a_t,h_t)
    &=\frac{p^\lambda(a_t\mid h_t, \lambda=1)}{p^\lambda(a_t\mid h_t, \lambda=0)}=\frac{p^\lambda(a_t\mid h_t, \lambda=1)p^\lambda(h_t,\lambda =1)}{p^\lambda(a_t\mid h_t, \lambda=0)p^\lambda(h_t,\lambda =0)}\nonumber\\
    &=\frac{g_t^\mathbbm{d}(a_t\mid h_t)p(h_t)}{g_t(a_t\mid h_t)p(h_t)}=\frac{f(a_t,h_t)}{p(a_t,h_t)}
\end{align}
We can also use algorithms directly aiming at marginal density ratios to estimate $r_t$. Again, in our density ratio super learner, we aim to include both kinds of estimators, i.e., both classification-based learners and kernel-based learners.

\subsection{Super Learning}
Consider the setting in which we observe independent and identically distributed copies $O_i,i=1,2,...,n$ of $O$. We use $P_n$ to denote the empirical probability distribution of $O_1, \dots, O_n$. For a target parameter $\psi_0(O)$ defined on some parameter space $\Psi$, we have a qualified loss function $L(O,\psi)$ for any element in this parameter space. The risk on the true data distribution is thus $E_0L(O,\psi)$. With the property that this risk will only be minimized when $\psi=\psi_0$, we can define the target parameter $\psi_0=\text{argmin}_{\psi\in\Psi}E_0L(O,\psi)$. For example, a binary classification problem ($\psi_0(O)=p_0(\lambda=1\mid X_1,X_2)$) can be formulated in this way by using the loss function $L(O,\psi)=-\mathbb{I}(\lambda=1)\log\psi(O)-\mathbb{I}(\lambda=0)(1-\log\psi(O))$. 

An estimator $\hat{\Psi}$ of $\psi_0$ is a mapping from a model space $\mathcal{M}$ to the parameter space $\Psi$. An estimate generated by this estimator based on some empirical distribution $P_n\in\mathcal{M}$ can be represented as $\hat{\Psi}(P_n)$. The risk of this estimate on a dataset with distribution $P$ is thus defined as
\begin{equation*}
R(\hat{\Psi}(P_n),P)=\int L\left\{o,\hat{\Psi}(P_n)\right\}dP(o)
\end{equation*}
For a $v$-fold cross-validation, we notate the dataset as follows. Let the folds be indexed by an index set $m\in\left\{1,2,...,v\right\}$, for each fold, we have an index set $T(m)\subset \left\{1,2,...,n\right\}$ for samples included in the test set $\left\{O_i:i\in T(m) \right\}$. The validation set can thus be represented as $\left\{O_i:i\in V(m) \right\}$ with $V(m)=T(m)^c$. Note that in $v$-fold cross-validation, every sample will and will only be present in one of the validation sets, i.e., we have $\cup_{m=1}^vV(m)=\left\{1,2,...,n\right\}$ and $V(m_1)\cap V(m_2)=\varnothing$ for $m_1\neq m_2$. Let $P_{T(m)}$ and $P_{V(m)}$ denote the empirical distributions of the training set and validation set in fold $m$, respectively. The cross-validated risk is defined as
\begin{equation*}
    \frac{1}{v}\sum_{m=1}^vR\left\{\hat{\Psi}(P_{T(m)}),P_{V(m)}\right\}=\frac{1}{v}\sum_{m=1}^v\int L\left\{o,\hat{\Psi}(P_{T(m)})\right\}dP_{V(m)}
\end{equation*}

For a library of selected learners, $\hat{\Psi}_k:k\in K$, super learning determines the weights for each learner by computing \begin{equation*}
    \hat{\beta}=(\hat{\beta}_1,...,\hat{\beta}_K)=\text{argmin}_{\beta}\frac{1}{v}\sum_{m=1}^vR\left\{\sum_{k\in K}\beta_k \hat{\Psi}_k(P_{T(m)}),P_{V(m)}\right\}
\end{equation*}
subject to the constraint that $\sum_{k\in K}\beta_k=1$ and $\beta_k\geq 0$ for all $k\in K$. The super learner is then defined as
\begin{equation*}
    \hat{\Psi}_\text{SL}(P_n)=\sum_{k\in K}\hat{\beta}_k \hat{\Psi}_k(P_n) \ .
\end{equation*}

\subsection{A Qualified Loss Function for Density Ratios}

Conventionally, for a parameter $\psi_0$, we regard an estimate $\hat{\psi}_1$ better than another estimate $\hat{\psi}_2$, if we have $E_0(L(O,\hat{\psi}_1))<E_0(L(O,\hat{\psi}_2))$. This judgment is based on the fact that $E_0(L(O,\psi))$ will and will only be minimized when $\psi=\psi_0$. For a density ratio target parameter $\psi_0(x_1,x_2)=\frac{p_0(x_1\mid x_2,\lambda=1)}{p_0(x_1\mid x_2, \lambda=0)}$, we now propose a new loss function satisfies this property. This loss function is defined as
\begin{equation}
    L(O,\psi)=-\mathbb{I}(\lambda=1)\log\psi(x_1,x_2)+\mathbb{I}(\lambda=0)\log\psi(x_1,x_2)
\end{equation}
\begin{thm}
    Suppose the marginal distributions of $X_1$ given $X_2$ and $\lambda$ have the same support for different values of $\lambda$, $p_0(\lambda=1)>0$, $p_0(\lambda=0)>0$.
    $L(O,\psi)=-\mathbb{I}(\lambda=1)\log\psi(x_1,x_2)+\mathbb{I}(\lambda=0)\log\psi(x_1,x_2)$. $E_0L(O,\psi)$ will only be minimized when $\psi(x_1,x_2)=\psi_0(x_1,x_2)$ a.s..
\end{thm}
\begin{proof}
    Let 
    \begin{equation*}
        \theta(\psi,x_1,x_2)=\frac{\psi(x_1,x_2)p_0(x_1,\mid x_2,\lambda=0)}{p_0(x_1\mid x_2,\lambda=1)}=\frac{p_0(x_1,\mid x_2,\lambda=0)}{p_0(x_1\mid x_2,\lambda=1)/\psi(x_1,x_2)}
    \end{equation*}
    \begin{align*}
         d(\psi,\psi_0)
        &=E_0L(O,\psi)-E_0L(O,\psi_0)\\
        &=\int -\mathbb{I}(\lambda=1)\log\theta(\psi,x_1,x_2)dP_0(o) +\int \mathbb{I}(\lambda=0)\log\theta(\psi,x_1,x_2) dP_0(o)\\
        &=\int -p_0(\lambda=1\mid x_1,x_2)\log\theta(\psi,x_1,x_2)dP_0(o)\\
        &\;\;\;\;+\int p_0(\lambda=0\mid x_1,x_2)\log\theta(\psi,x_1,x_2) dP_0(o)\\
        &=p_0(\lambda=1)\int -p_0(x_1,x_2\mid \lambda=1)\log\theta(\psi,x_1,x_2)dx_1dx_2\\
        &\;\;\;\;+p_0(\lambda=0)\int p_0(x_1,x_2\mid \lambda=0)\log\theta(\psi,x_1,x_2)dx_1dx_2\\
        &=p_0(\lambda=1)\int p_0(x_2\mid \lambda=1)\int-p_0(x_1\mid x_2, \lambda=1)\theta(\psi,x_1,x_2)dx_1dx_2 \\
        &\;\;\;\;+p_0(\lambda=0)\int p_0(x_2\mid \lambda=0)\int p_0(x_1\mid x_2, \lambda=0)\log\theta(\psi,x_1,x_2)dx_1dx_2\\
        &=p_0(\lambda=1)E_{0\mid \lambda=1}\left[\text{KL}\left\{P_{0\mid X_2,\lambda=1}|| P_{\psi\mid X_2,\lambda=1}\right\}\right]\\
        &\;\;\;\;+p_0(\lambda=0)E_{0\mid \lambda=0}\left[\text{KL}\left\{P_{0\mid X_2,\lambda=0}|| P_{\psi\mid X_2,\lambda=0}\right\}\right]\\
        &\geq0
    \end{align*}
    Here we use $P_{0\mid X_2,\lambda}$ to denote the marginal distribution of $X_1$ given $X_2$ and $\lambda$. The corresponding probability density function is $p_0(x_1\mid X_2,\lambda)$. On the other hand, $P_{\psi\mid X_2,\lambda=1}$ represents a marginal distribution of $X_1$ with a probability function $p_0(x_1\mid X_2,\lambda=0)\psi(x_1,X_2)$. Similarly, $P_{\psi\mid X_2,\lambda=0}$ has the probability density function $p_0(x_1\mid X_2,\lambda=1)/\psi(x_1,X_2)$. We have that $d(\psi,\psi_0)=0$ if and only if $p_0(x_1\mid X_2,\lambda=0)\psi(x_1,X_2)=p_0(x_1\mid X_2, \lambda=1)$ and $p_0(x_1\mid X_2,\lambda=1)/\psi(x_1,X_2)=p_0(x_1\mid X_2, \lambda=0)$ a.s., i.e., $\psi=\psi_0$ a.s..
\end{proof}

Given this theorem, with an additional assumption that $\big|L(O,\psi)\big|<C$ for some constant $C<\infty$, a density ratio super learner based on this loss function will enjoy the oracle inequality established in \cite{laan_super_2007,dudoit_asymptotics_2005}.

\subsection{Simulation Study}

We conducted two numerical studies corresponding to the two causal inference scenarios stated in Section \ref{dr_in_ci}, using Monte Carlo simulation. 

In each of the simulations, we construct a density ratio super learner to estimate the nuisance density ratio parameter. The performance of our super learner is compared to that of a baseline estimator using their respective average hold-out risks, which are calculated based on the new loss function we introduce. The hold-out risk of an estimate is approximated numerically, by computing the empirical risk on a large hold-out dataset independent of the training dataset. Smaller hold-out risks imply superior estimates. Formally, suppose we have a hold-out dataset with empirical distribution $P_H$, an estimator $\hat{\Psi}$, we generate multiple training sets $P_{T_i},i\in\left\{1,\ldots,k\right\}$, the average hold-out risk of this estimator is given by $\frac{1}{k}\sum_{i=1}^kR(\hat{\Psi}(P_{T_i}),P_H)$.

In our simulations, we include both kernel- and classification-based learners as candidate learners in our super learners. In the mediation analysis setting, a classification-based learner has to contain two estimators to estimate $p(A\mid M,W)$ and $p(A\mid W)$ respectively, see Equation \ref{2_cl}. In our simulation, we build a classification-based learner using two classification super learners to estimate the probabilities. This learner is included in our final density ratio super learner and is made our baseline estimator. While in the LMTP simulation, a classification-based super learner contains only one classification estimator. We build a classification super learner as our baseline estimator, as this is adopted by other researchers \cite{diaz_nonparametric_2021}. The kernel-based learners are built on the KLIEP and RuLSIF algorithms \cite{sugiyama_dimensionality_2010, sugiyama_direct_2007, kanamori_least-squares_2009}.

\section{Results}

\subsection{Mediation Analysis Setting}

We generated 100 training datasets with sample size $n$ for each $n\in\left\{100,500,1000,2000\right\}$, and a hold-out dataset with sample size $10000$ from the following data generating mechanism:
\begin{align*}
    W &\sim N_{T_1} \mid \text{where $N_{T_1}$ is $N(5,4)$ truncated to the interval $\left[2,8\right]$} \\
    A\mid W &\sim \text{Bernoulli}\left\{0.6-0.35\mathbb{I}(W<4)-0.15\mathbb{I}(W>5)\right.\\
    &\;\;\;\;\;\;\;\;\;\;\;\;\;\;\;\;\;\;\;\left.+0.05\mathbb{I}(W<6)-0.15\mathbb{I}(W>7)\right\}\\
    M\mid W, A=1 &\sim \text{Beta}(0.6W+1, 0.7W)\\
    M\mid W, A=0 &\sim N_{T_W} \text{where $N_{T_W}$ is $N(0.1W,1)$ truncated to the interval $\left[0,1\right]$} \\
\end{align*}
The density ratio parameter of interest in this setting is $\psi_0(m,w)=\frac{p(m\mid w,A=1)}{p(m\mid w,A=0)}$. We build a density ratio super learner with a library of three kernel-based learners and one classification-based learner. The classification-based learner uses two classification super learners to estimate $p(A\mid,M,W)$ and $p(A\mid,W)$. This learner is referred to as the classification SL learner. We show the results of this simulation in Table \ref{tab:holdout_mediation} and Figure \ref{fig:holdout_mediation}. In Figure \ref{fig:holdout_mediation}, four sub-figures show the results under each sample size. In each of the sub-figures, learners are in ascending order based on their average risk-out risks. In the line of each learner, a single point represents the hold-out risk of an estimate generated with one training set, and the red cross represents the average hold-out risks over all the estimates based on different training sets.

\begin{table}[]
    \centering
    \begin{tabular}{c c c c c c}
        \hlineB{5}
         Learners $\backslash$ Sample Size & $n=100$ & $n=500$ & $n=1000$ & $n=2000$ \\
         \hlineB{3}
         Kernel-Based Learner1 & -0.164 & -0.051 & -0.045 & -0.039\\ 
         Kernel-Based Learner2 & -0.215 & -0.394 & -0.238 & -0.221\\
         Kernel-Based Learner3 & -0.234 & -0.412 & -0.243 & -0.209\\
         Classification SL Learner & -0.235 & -0.329 & -0.355 & -0.371\\
         Density Ratio Super Learner & -0.254 & -0.402 & -0.326 & -0.371\\ 
         \hlineB{3}
    \end{tabular}
    \caption{Average Hold-Out Risks For Individual Learners and the Super Learner}
    \label{tab:holdout_mediation}
\end{table}

\begin{figure}
    \centering
    \includegraphics[width=4.5in]{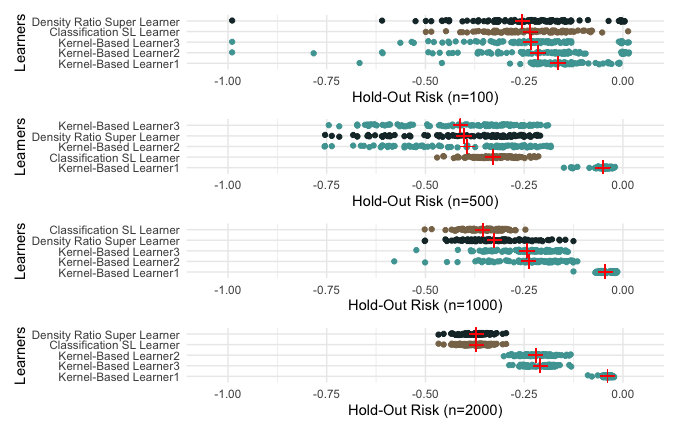}
    \caption{Average Hold-Out Risks For Individual Learners and the Super Learner}
    \label{fig:holdout_mediation}
\end{figure}

We observe that when the sample size is small, the density ratio super learner will outperform the classification SL learner. As the sample size increases, the classification SL learner will perform better than other kernel-based learners and the density ratio super learner is fully decided by the classification SL learner.

\begin{figure}
    \centering
    \includegraphics[width=4.5in]{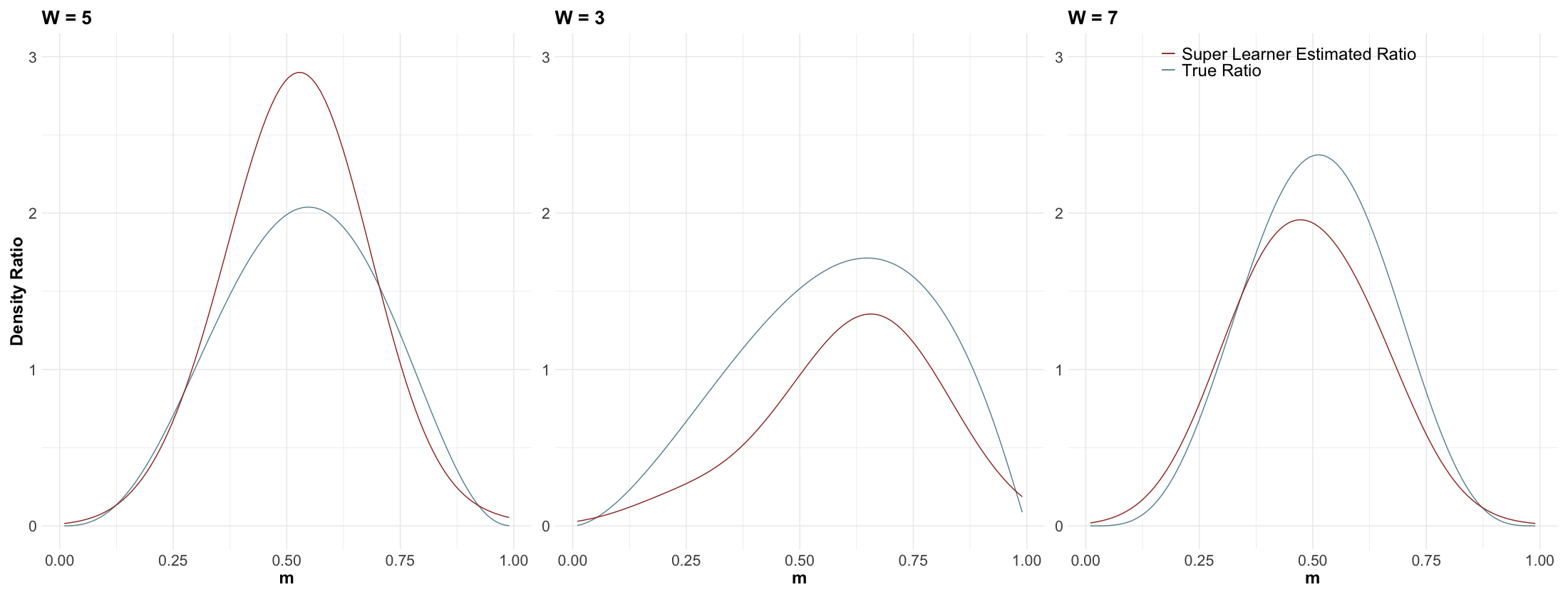}
    \caption{True Ratio vs Super Learner Estimated Ratio at Different Values of $W\;(n=700)$}
    \label{fig:profile_n700}
\end{figure}

\begin{figure}
    \centering
    \includegraphics[width=4.5in]{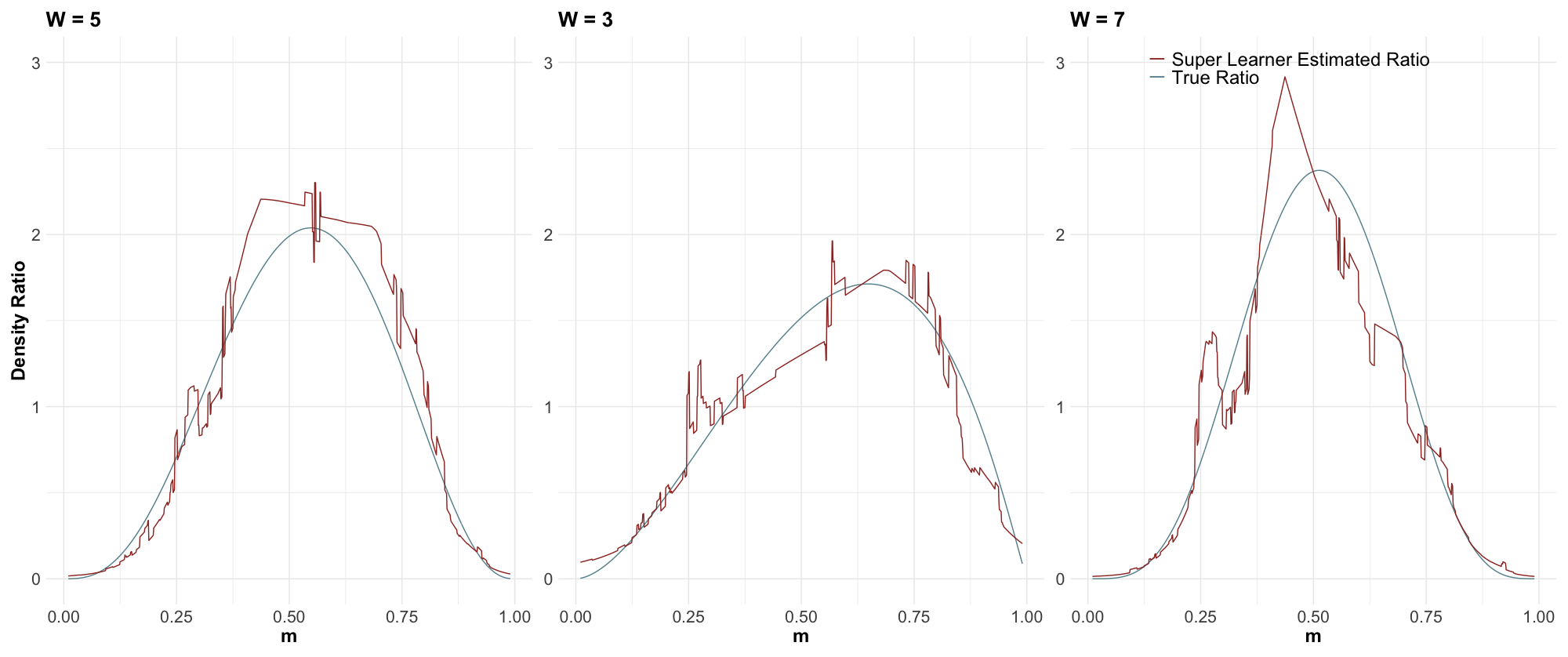}
    \caption{True Ratio vs Super Learner Estimated Ratio at Different Values of $W\;(n=2000)$}
    \label{fig:profile_n2000}
\end{figure}

Figure \ref{fig:profile_n700} and Figure \ref{fig:profile_n2000} give profiles of true ratios and density ratios estimated by the super learner with different values of $m$ and $w$ in one trial with $n=700$ and $n=2000$, respectively. We can see that as the sample size increases, the estimates given by the density ratio super learner tend to have smaller biases but larger variances locally.

\subsection{LMTP Setting}

We generated 100 training datasets with sample size $n$ for each $n\in\left\{100,500,1000,2000\right\}$, and a hold-out dataset with sample size $10000$ from the following data generating mechanism that is used in \cite{diaz_nonparametric_2021} :
\begin{align*}
    W_1 &\sim \text{Cat}(0.5,0.25,0.25)\\
    A_1\mid W_1 &\sim \text{Binomial}\left\{5,0.5\mathbb{I}(W_1>1)+0.1\mathbb{I}(W_1>2)\right\}\\
    W_t\mid(\bar{A}_{t-1},\bar{W}_{t-1})&\sim\text{Bernoulli}\left\{\text{expit}(-0.3L_{t-1}+0.5A_{t-1})\right\}\;\text{for}\;t\in \left\{2,3,4\right\}\\
    A_t\mid(\bar{A}_{t-1},\bar{W}_t)&\sim \left\{
    \begin{array}{ll}
         &\text{Binomial}\left\{5,\text{expit}(-2+1/(1+2W_t+A_{t-1}))\right\}\;\text{for}\;t\in\left\{2,3\right\}  \\
         &\text{Binomial} \left\{5,\text{expit}(1+W_t-3A_{t-1})\right\}\;\text{for}\;t=4
    \end{array}
    \right.
\end{align*}
with the longitudinal modified treatment policy
\begin{align*}
    \mathbbm{d}(a_t,h_t)=\left\{
        \begin{array}{cll}
             &  a_t-1\;\; &\text{if}\;a_t\geq 1\\
             & a_t &\text{if}\;a_t<1
        \end{array}
    \right.
\end{align*}
The density ratio parameters in this setting that need to be estimated are $r_t(a_t,h_t)=\ \frac {g_t^\mathbbm{d}(a_t\mid h_t)}{g_t(a_t\mid h_t)},t\in\left\{1,2,3,4\right\}$. We adopt the data-augmenting framework described in Section \ref{method:lmtp} to conduct the estimation process. For each $r_t$, we build a density ratio super learner containing kernel- and classification-based learners, and another plug-in estimator that uses a classification super learner to estimate the conditional odd as our baseline estimator. The results are presented in Table \ref{tab:holdout_lmtp}. We observe that the density ratio super learner performs better than the baseline estimator in all cases except for estimating $r_1$ in the $n=2000$ setting.

\begin{table}[]
    \centering
    \begin{tabular}{c c c c c c c}
        \hlineB{5}
        Sample Sizes & Learners $\backslash$ Estimand & $r_1$ & $r_2$ & $r_3$ & $r_4$ \\
        \hlineB{3}
        \multirow{2}{4em}{$n=100$} & Density Ratio SL & -0.329 & -0.719 & -0.774 & -0.570 \\
        & Classification SL & -0.303 & -0.638 & -0.659 & -0.222\\
        \multirow{2}{4em}{$n=500$} & Density Ratio SL & -0.316 & -0.673 & -0.909 & -0.989 \\
        & Classification SL & -0.304 & -0.644 & -0.666 & -0.473\\
        \multirow{2}{4em}{$n=1000$} & Density Ratio SL & -0.329 & -0.714 & -0.711 & -1.057 \\
        & Classification SL & -0.318 & -0.657 & -0.674 & -0.622\\
        \multirow{2}{4em}{$n=2000$} & Density Ratio SL & -0.348 & -0.721 & -0.697 & -1.135 \\
        & Classification SL & -0.351 & -0.672 & -0.686 & -0.769\\
        \hlineB{3}
    \end{tabular}
    \caption{Average Hold-Out Risks For the Two Estimators}
    \label{tab:holdout_lmtp}
\end{table}

\section{Discussion and Conclusion}

In this study, we propose a novel loss function tailored specifically for generating super learner estimates of density ratios. We show that this loss function is qualified for building super learners. Through two simulations, we empirically demonstrate how a density ratio super learner would perform in estimating density ratios in mediation analysis and longitudinal modified treatment policy settings. Our results indicate that the density ratio super learner would asymptotically minimize the risk in comparison to its candidate individual learners. Our results also show that the density ratio super learners could achieve more reliable performance compared to other baseline estimators used by past researchers.

Although our study focuses on causal inference scenarios, it is notable that the super learning approach can also be applied to estimating density ratios in other fields. For example, our density ratio super learner can be used to tackle the covariate shift problem in general machine learning.

% BIBLIOGRAPHY
\bibliographystyle{unsrt}
\bibliography{reference}

\end{document}